\documentclass[11pt]{article}
\usepackage[utf8]{inputenc}

\title{Concentration Bounds for Discrete Distribution Estimation in KL Divergence}

\usepackage{xspace}
\usepackage{amsfonts,amsmath,amssymb,amsthm, pbox}

\usepackage{multirow}
\usepackage{dsfont} %

\usepackage{algorithmic,algorithm}

\usepackage[shortlabels]{enumitem}
\setitemize{noitemsep,topsep=0pt,parsep=0pt,partopsep=0pt}
\setenumerate{noitemsep,topsep=0pt,parsep=0pt,partopsep=0pt}

\makeatletter
\@ifundefined{theorem}{%
  \theoremstyle{definition}
  
  \theoremstyle{plain}
  \newtheorem{theorem}{Theorem}
  \newtheorem{corollary}{Corollary}
  \newtheorem{lemma}{Lemma}
  
  \newtheorem{fact}{Fact}
  
  \theoremstyle{remark}
  \newtheorem{remark}{Remark}

}{}
\makeatother

\newcommand{\ignore}[1]{}

\newcommand{\EE}{\mathbb{E}}

\newcommand{\RR}{\mathbb{R}}

\newcommand{\expectation}[1]{\EE\left[#1\right]}

\newcommand{\Var}{{\rm Var}}

\newcommand{\eg}{\textit{e.g.,}\xspace}
\newcommand{\ie}{\textit{i.e.,}\xspace}  %
\newcommand{\iid}{\textit{i.i.d.}} %

\def \Paren#1{{\left({#1}\right)}}

\def \frc#1#2{{\frac{#1}{#2}}}

\newcommand{\eqdef}{{:=}}

\def\ignore#1{}

\newcommand{\bi}{\begin{itemize}}
\newcommand{\ei}{\end{itemize}}

\def\orpro{\mathop{\mathchoice
   {\vee\kern-.49em\raise.7ex\hbox{$\cdot$}\kern.4em}
   {\vee\kern-.45em\raise.63ex\hbox{$\cdot$}\kern.2em}
   {\vee\kern-.4em\raise.3ex\hbox{$\cdot$}\kern.1em}
   {\vee\kern-.35em\raise2.2ex\hbox{$\cdot$}\kern.1em}}\limits}

\def\andpro{\mathop{\mathchoice
 {\wedge\kern-.46em\lower.69ex\hbox{$\cdot$}\kern.3em}
 {\wedge\kern-.46em\lower.58ex\hbox{$\cdot$}\kern.25em}
 {\wedge\kern-.38em\lower.5ex\hbox{$\cdot$}\kern.1em}
 {\wedge\kern-.3em\lower.5ex\hbox{$\cdot$}\kern.1em}}\limits}

\def\simge{\mathrel{%
   \rlap{\raise 0.511ex \hbox{$>$}}{\lower 0.511ex \hbox{$\sim$}}}}

\def\simle{\mathrel{
   \rlap{\raise 0.511ex \hbox{$<$}}{\lower 0.511ex \hbox{$\sim$}}}}

\DeclareMathOperator*{\KL}{KL}
\DeclareMathOperator*{\var}{Var}

\newcommand{\ab}{k}

\newcommand{\ns}{n}

\newcommand{\eps}{\varepsilon}

\usepackage[dvipsnames]{xcolor}
\usepackage[capitalise]{cleveref}

\newcommand{\Bin}{\text{Bin}}
\newcommand{\emp}{\text{emp}}

\newcommand{\indic}[1]{\mathbf{1}_{#1}}
\newcommand{\nzero}{n_0}

\renewcommand{\eqdef}{\triangleq}

\DeclareMathOperator{\poisson}{Poisson}
\usepackage[numbers]{natbib}
\usepackage{graphicx}
\usepackage{fullpage}
\usepackage{algorithm, algorithmic}

\usepackage{authblk}
\author[1]{Cl\'ement L. Canonne}
\author[2]{Ziteng Sun}
\author[2]{Ananda Theertha Suresh\thanks{theertha@google.com}}
\affil[1]{University of Sydney}
\affil[2]{Google Research, New York}

\newcommand{\conf}[1]{}
\newcommand{\arxiv}[1]{#1}

\begin{document}

\maketitle

\begin{abstract}
We study the problem of discrete distribution estimation in KL divergence and provide  concentration bounds for the Laplace estimator. We show that the deviation from mean scales as $\sqrt{k}/n$ when $n \ge k$, improving upon the best prior result of $k/n$. We also establish a matching lower bound that shows that our bounds are tight up to  polylogarithmic factors.
\end{abstract}

\section{Introduction}

Discrete distribution estimation, \ie density estimation over discrete domains, is a fundamental problem in Statistics, with a rich history (see, \eg~\cite{DevroyeL01,Diakonikolas16} for an overview and further references). In this work, we address a simple yet surprisingly ill-understood aspect of this question: what is sample complexity of estimating an arbitrary discrete distribution in Kullback–Leibler ($\KL$) divergence \emph{with vanishing probability of error?}

To describe the problem further, a few definitions are in order. Let $\Delta_\ab$ denote the set of probability distributions over $[\ab] = \{1,2,3 \ldots, \ab\}$, \ie
\[
\Delta_\ab \eqdef \left\{p : \sum^\ab_{i=1} p_i = 1, p_i \geq 0 \, \forall i \in [\ab] \right\},
\]
where we identify a discrete probability distribution with its probability mass function (pmf), seen as a vector. 
Given $\ns$ independent samples $X^\ns \eqdef X_1, X_2, \ldots, X_\ns$ from an unknown distribution $p \in \Delta_k$, the goal of distribution estimation is to estimate the underlying distribution $p$. The accuracy of the estimate is typically measured using a loss function $L\colon \Delta_k \times \Delta_k \to \RR^+ \cup \{0\}$. Popular loss functions include $\ell_1$ (or, equivalently, total variation (TV) distance), $\ell_2$, and $\ell_\infty$ distances, as well as Kullback--Leibler (KL). Specifically, for $r \geq 1$, the $\ell_r$ distance is given by  
\[
\ell_r(p, q) = \left( \sum^k_{i=1} |p_i - q_i|^r \right)^{1/r},
\]
and the KL divergence is defined as
\[
\KL(p \,\|\,  q) = \sum^k_{i=1} p_i \log \frac{p_i}{q_i}.
\]
where $\log$ denotes the natural logarithm, and with the convention $0\log 0 = 0$. Note that KL divergence, while always non-negative, is unbounded and asymmetric, and in general $\KL(p \,\|\,  q) \neq \KL(q \,\|\,  p)$; moreover, it is related to $\ell_1$ distance \emph{via} (among others) Pinsker's inequality, which states that 
$
\ell^2_1(p,q) \leq 2\KL(p \,\|\,  q)
$ for every $p,q\in\Delta_\ab$. For distribution estimation questions under KL loss, one typically uses $L(p,\hat{p}) \eqdef \KL(p \,\|\,  \hat{p})$, that is, the estimate is the second argument. 

\paragraph{Commonly used estimators\arxiv{.}} 
Perhaps the most used estimator is the maximum likelihood estimator or \emph{empirical estimator}, which estimates the distribution as $\hat{p}^\emp$ defined by  
\begin{equation}
\hat{p}^\emp_i = \frac{N_i}{n}, \qquad i\in[\ab],
\end{equation}
where $N_i$ is the number of occurrences of symbol $i$ in $X^n$. While the empirical estimator has good convergence guarantees in $\ell_r$ distances, $\KL(p \,\|\,  \hat{p}^\emp)$ can become unbounded, since $\hat{p}^\emp_i$ can be zero for some $i$ with $p_i>0$. To avoid this, \emph{add-constant estimators} are used to obtain estimators in KL divergence, where the \emph{add-$t$ estimator} is given by
\begin{equation}
\hat{p}^t_i = \frac{N_i+ t}{n + k \cdot t}, \qquad i\in[\ab].
\end{equation}
One can easily check that this defines a \textit{bona fide} probability distribution. Intuitively, the parameter $t$ allows one to smooth the estimator to interpolate between the empirical one and a uniform prior, thus mitigating the previous issue. 
The add-constant estimator with $t=1$ is commonly referred to as the \emph{Laplace estimator}, while the add-constant estimator with $t=1/2$ is also known as the \emph{Krichevsky--Trofimov (KT) estimator} \cite{krichevsky1981performance}.

\paragraph{Minimax rates and concentration bounds\arxiv{.}} Typically these estimators are studied to provide \emph{minimax rates}, which bound the expected loss
\[
\min_{\hat{p}}\max_{p \in \Delta_k} \EE_{X^n \sim p^n} \left[ L(p, \hat{p}(X^n)) \right]\,,
\]
where the minimum is taken over all possible $\ns$-sample estimators $\hat{p}\colon [\ab]^\ns \to \Delta_\ab$, and the maximum over all possible discrete distributions on $[\ab]$. 
The precise asymptotic minimax rates (including the leading constant) for a fixed $k$ and as $n$ increases are known for $\ell_1$ distance \cite{kamath2015learning}, $\ell_2$ distance \cite[p.~349]{lehmann2006theory},  scaled $\ell_2$ losses and chi-squared type distances \cite{olkin1979admissible, rutkowska1977minimax, kamath2015learning}, as well as KL divergence \cite{braess2004bernstein}. Distribution-specific rates are also known for $\ell_1$ distance \cite{han2015minimax,Canonne:NoteLearningDistributions,CohenKW20}.
In terms of $\KL(p \,\|\,  \hat{p}(X^\ns))$, \cite{mourtada2022improper} showed that the rate of  convergence for the Laplace estimator is given by
\begin{equation}
    \label{eq:minimax:laplace:kl}
\max_{p \in \Delta_k} \EE_{X^n \sim p} \left[ \KL(p \,\|\,  \hat{p}^1) \right] \leq \log \frac{n+k}{n+1} \leq \frac{k-1}{n},
\end{equation}
and where the upper bound is asymptotically tight as $n\to\infty$ (for fixed $k$).

As opposed to the expectation bounds above, we wish to provide \emph{concentration bounds} of the following form. With probability $1-\delta$,
\begin{equation}
L(p, \hat{p}) \leq \EE[L(p, \hat{p})] + t_\delta.
\end{equation}
Such a bound for $\ell_1$ distance follows straightforwardly from an application of McDiarmid's inequality; an extension for arbitrary discrete domains (unbounded $k$) can be found in \cite{cohen2020learning}. For general $\ell_r$ losses, the set of results is surveyed in~\cite{Canonne:NoteLearningDistributions}.

For KL divergence, \cite{BhattacharyyaGPV21} showed that for the Laplace estimator $\hat{p}^1$, with probability at least $1-\delta$,
\begin{equation}
    \label{eq:bound:bgpv}
\KL(p \,\|\,  \hat{p}^1) \leq \EE[\KL(p \,\|\,  \hat{p}^1)] + t_{\delta}
\end{equation}
where $t_{\delta} \asymp \frac{k}{n}\cdot \log n\cdot \log\frac{k}{\delta}$ (where $\asymp$ denotes equality up to constants).
However, note that in view of~\eqref{eq:minimax:laplace:kl} above, this shows that the $t_{\delta}$ term dominates the expectation term, raising the question of whether one can strengthen this, and decouple the bound into an expectation term and a (mild) concentration term $t_\delta$. Recently, \cite[Lemma 17]{han2021optimal} showed that for the Laplace estimator $\hat{p}^1$, with probability at least $1-\delta$,
\begin{equation}
    \label{eq:bound:bgpv2}
\KL(p \,\|\,  \hat{p}^1) \leq \frac{2k}{n} + \tilde{t}_{\delta},
\end{equation}
where $\tilde{t}_{\delta} \asymp \frac{\sqrt{k}}{n}\cdot \bigl(\log \frac{1}{\delta}\bigr)^3$. While this improves on earlier work, the bound does not have the true expectation $\EE[\KL(p \,\|\,  \hat{p}^1)]$, but an upper bound $ \frac{2k}{n}$, which is at least a factor $2$ off (see \eqref{eq:minimax:laplace:kl}). 

Our main result is to show that one can get tight bounds similar to \eqref{eq:bound:bgpv2}, but with $\EE[\KL(p \,\|\,  \hat{p}^1)]$ as the leading term (and a slight improvement in the exponent of the logarithm). Namely, we show that the following bound on the minimax rate of the Laplace estimator holds:
\begin{theorem}[Main theorem, Informal (\cref{thm:KL})]
\label{thm:KL:informal}
With probability at least $1-\delta$, the Laplace add-1 estimator on $\ns\geq \ab$ samples satisfies
\[
\KL(p \,\|\,  \hat{p}^1 ) \leq \EE[\KL(p \,\|\,  \hat{p}^1 ) ] + O\!\left(\frac{\sqrt{k}}{n}\log^{5/2} \frac{k}{\delta} \right).
\]
\end{theorem}
\noindent We refer the reader to~\cref{thm:KL} for the detailed statement, including leading constants and a bound holding for all regimes of $\ns$. We emphasize that, in contrast to the previous bound given in~\eqref{eq:bound:bgpv}, the additional $t_\delta$ term is here \emph{sublinear in $\ab$}, and in particular negligible in front of the expectation term $ \EE[\KL(p \,\|\,  \hat{p}^1 ) ]$ for most values of $\delta$. Viewed differently, our result improves on that of~\cite{BhattacharyyaGPV21} for all $\delta \geq \exp(-\tilde{O}(k^{1/3}))$.

Given the above result, it is natural to wonder if our bound can be significantly improved, and in particular if the $\sqrt{\ab}$ dependence in the $t_\delta$ term can be improved upon. In~\cref{sec:variance:lb}, we answer this negatively, by establishing a lower bound on the variance of $\KL(p \,\|\,  \hat{p}^1 )$ (\cref{th:variance:lb}). By a standard argument, this lower bound implies that the $\sqrt{\ab}$ dependence in the concentration bound is essentially optimal (\cref{cor:sqrtk:lb}). We also show in \cref{fig:samplevsestimate} that the dependence of $\sqrt{k}/n$ on the variance of the Laplace estimator can be observed in simulated experiments. We generate 10240 \iid~samples from uniform distributions of different support sizes and compute the standard deviation of $ \EE[\KL(p \,\|\,  \hat{p}^1 ) ]$ over 1000 repetitions and compare it with the estimated standard deviation of $\sqrt{\frac{k}{2}}\frac1{\ns}$.\footnote{See Appendix~\ref{app:figure} for a heuristic explanation of the constant.} As we can see, the estimate approximates the true sample standard deviation well.

\begin{figure}[t]
\centering
\includegraphics[width = \conf{0.45}\arxiv{0.75}\textwidth]{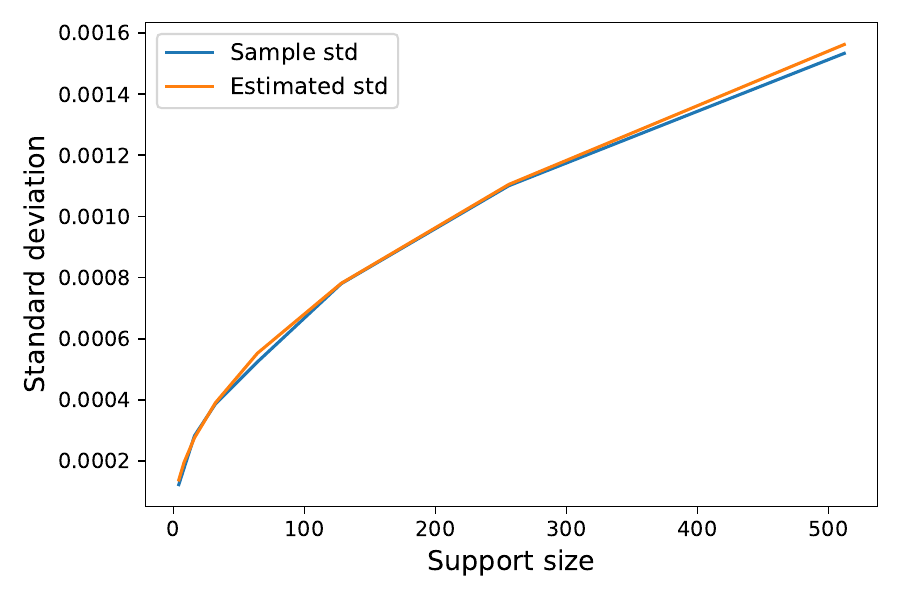}
\caption{Sample standard deviation vs. estimated standard deviation of the Laplace estimator where the sample mean is computed over $10240$ \iid~samples generated from uniform distributions with different support sizes. Each standard deviation is computed over 1000 independent experiments.}
\label{fig:samplevsestimate}
\end{figure}

\begin{remark}
We emphasize that our result applies to the loss $L(p,\hat{p}) = \EE[\KL(p \,\|\,  \hat{p})]$, and not to the (different) goal of minimizing $\EE[\KL(\hat{p} \,\|\,  p)]$. This latter goal is much easier, in the sense that the empirical estimator not only provides non-trivial bounds for it, but is known to achieve the optimal (up to constant factors) rate, as well as to provide (similarly optimal) high-probability bounds~\cite{Agrawal20,menard2021fast,Bhatt22sharp}. We note that~\cite{MardiaJTNW20} provides upper and lower bounds on the variance of $\EE[\KL(\hat{p} \,\|\,  p)]$ for the empirical estimator.
\end{remark}

To conclude this introduction, we mention a direct consequence of our result: plugging our main result,~\cref{thm:KL:informal}, into the proof of~\cite[Theorem~1.4]{BhattacharyyaGPV21} yields an improved sample complexity (to achieve $\EE[\KL(p \,\|\, \hat{p})] \le \eps$)
\[
    n \asymp \frac{d k^2}{\eps} + \frac{d k^{3/2}}{\eps} \log^{5/2} \frac{d k}{\delta} + \frac{1}{\eps}\log\frac{1}{\eps}\log\frac{d}{\delta}
\]
instead of the previous sample complexity bound of 
\[
    n \asymp \frac{d k^2}{\eps}\log \frac{d k}{\delta} \left(\log \frac{d k}{\eps}  + \log\log\frac{1}{\delta} \right)
\]
for the problem of estimating to KL divergence $\eps$ a tree-structured Bayesian network (with known structure) with $d$ nodes over alphabet of size $k$, and failure probability $\delta$. While the expressions are quite unwieldy, this yields a significant improvement for large alphabets ($k\gg 1$).

\section{Analysis sketch}
A standard way to provide concentration bounds for discrete distribution estimators is to use McDiarmid's inequality, which applies to functions of independent random variables satisfying a ``bounded difference'' condition.
\begin{lemma}[McDiarmid's inequality]
Let $X^n \triangleq X_1, X_2,\ldots, X_n$ be $n$ independent random variables and $f\colon X^n\to \RR$. Suppose changing  $X_i$ changes the absolute value of the function by at most $c_i$; then, with probability at least $1-\delta$,
\[
f(X^n) \leq \EE[f(X^n)] + \sqrt{\frac{\sum^n_{i=1} c^2_i}{2} \log \frac{1}{\delta}}.
\]
\end{lemma}
Letting $c \eqdef \max_{i} c_i$, the above bound simplifies to
\begin{equation}
    \label{eq:mcdiarmid:maincase}
f(X^n) \leq \EE[f(X^n)] + \sqrt{\frac{n c^2}{2} \log \frac{1}{\delta}}.
\end{equation}
We use this bound in most of the paper and refer to $c$ as $c_\infty(f)$. In view of the above, the goal is to obtain a good enough bound on $c_\infty(f)$ to show that $n c_\infty(f)^2$ decreases to $0$, so that the additive error term can be made as small as desired by choosing $\ns$ large enough (as a function of $\delta$). For instance, \cite{Canonne:NoteLearningDistributions} observes that the $\ell_1$ distance between the true distribution and the empirical distribution satisfies $c_\infty\leq 2/n$ and uses it to obtain a concentration bound.

Unfortunately, a direct application of McDiarmid's inequality for $\KL(p \,\|\,  \hat{p}^1)$ results in a vacuous bound, as $c_\infty$ is a constant (and so $n c_\infty(f)^2 \nearrow \infty$). To illustrate this, let $k=2$, and consider the distribution $p = (1/2, 1/2)$, $N_1 = n$ and $N_2 = 0$. Then  $\KL(p \,\|\,  \hat{p}^1) = \log(n+2) + \log \frac{1}{2} - \frac{1}{2} \log(n+1)$. Changing one sample can change counts to $N'_1 = n -1 $ and $N'_2 = 1$, in which case the KL divergence becomes
$  \log(n+2) + \frac{1}{2} \log\frac{1}{2} - \log n  - \frac{1}{2} \log(2)$. Hence,
\[
c_\infty(\KL) \geq \frac{1}{2} \log 2 + \frac{1}{2} \log \frac{n}{n+1} \operatorname*{\to}_{n \to \infty} \frac{1}{2} \log 2.
\]
To provide a stronger bound, we first note that the KL divergence can be written as a function of $k$ \emph{counts} $N_1, N_2, \ldots, N_k$, instead of the $\ns$ samples $X_1,\dots, X_\ns$. If we can potentially bound how much the KL divergence changes if we change each $N_i$, we may be able to apply McDiarmid's inequality to this different parameterization and potentially get a different bound.

However, this approach suffers from multiple problems. Firstly, the counts $N_i$s are not independent of each other as $\sum^k_{i=1} N_i = n$. To overcome this, we relate the standard multinomial sampling process to the Poisson sampling process \cite{mitzenmacher2017probability}, where instead of obtaining $n$ samples, we obtain $N$ samples from the distribution $p$, where $N$ is an independent Poisson random variable with parameter $n$.  It is well known that under this sampling, the counts $N_i$ are independent of each other \cite{mitzenmacher2017probability}.

The second question that arises is how to define the estimator with $N\sim\poisson(n)$ samples. A natural way to do the same is to use the estimator $\hat{p} = \frac{N_i + 1}{N+k}$. However, with this definition, it is difficult to bound the resulting $c_\infty$ for KL divergence as changing any of the counts changes the denominator term $N$. Hence, we relate the analysis to another pseudo-estimator $\frac{N_i + 1}{n+k}$. Note that this estimate is not a probability distribution as the sum of estimated probabilities is not guaranteed to sum to one.

Finally, even with these changes $c_\infty(\KL)$ can be large as illustrated in the above example. In other words, in the worst case  $c_\infty(\KL)$ \emph{can} be a constant. However, in the above example, even though $p_1 = 1/2$ we took $N_1 \approx n$, which is highly unlikely to happen. We would like to somehow use the fact that ``typically'' $N_i \approx n p_i$ when computing the $c_\infty(\KL)$. To this end, we define the following quantity
for two measures $p$ and $q$ (which need not sum to one):  
\begin{align}
\widetilde{\KL}(p \,\|\,  q) \conf{&}= \KL(p \,\|\,  q) + \frac{n+k}{n}\sum_{i=1}^k q_i 
+  \sum_{i=1}^k p_i \log\frac{n}{e(n+k)}.
\label{eq:tildeKL}
\end{align}
Note that when $p,q$ are \textit{bona fide} probability measures, then $\widetilde{\KL}(p \,\|\,  q)$ is just $\KL(p \,\|\,  q)$ shifted by a (deterministic) quantity depending only on $n$ and $k$.  
With these modifications, we prove our results in the following three steps.

\begin{enumerate}
\item Show that this new quantity $\widetilde{\KL}$ has similar expectations under multinomial and Poisson sampling (Lemma~\ref{lem:KL_expectation}).
\item Using this, show that a concentration bound for $\widetilde{\KL}$ under Poisson sampling yields a concentration bound for $\KL$ under multinomial sampling (Lemma~\ref{lem:mult_to_poisson}).
\item Provide a concentration bound for $\widetilde{\KL}$ under Poisson sampling by proving a high probability bound on
$c_\infty(\widetilde{\KL})$.
\end{enumerate}
In the rest of the section, we state the above results formally. \conf{Due to space constraints, some of the proofs are deferred to the appendix. }\arxiv{The proofs are deferred to the appendix.}
\begin{lemma}
\label{lem:KL_expectation}
Let $\hat{p}'$ be the measure given by $\hat{p}'_i = \frac{N'_i + 1}{n + k}$, where $N'_1, \ldots, N'_k \sim \text{Mult}(N, p)$ and $N \sim \operatorname{Poi}(n)$ then
\[
\EE \left[\widetilde{\KL}(p \,\|\,  \hat{p}')\right]  
\leq \EE \left[\widetilde{\KL}(p \,\|\,  \hat{p}^1) \right] + \frac{311}{n} + \frac{160k}{n^{3/2}} .
\]
\end{lemma}
We defer the detailed proof of \cref{lem:KL_expectation} to \cref{app:KL_expectation}. At a high level we show that $\EE \left[\widetilde{\KL}(p \,\|\,  \hat{p}')\right]  - \EE \left[\widetilde{\KL}(p \,\|\,  \hat{p}^1) \right]$ can be written as
$\sum^k_{i=1} p_i \EE \left[ \log (N_i + 1) \right] -  \EE\left[ \log (N'_i + 1)  \right] $. By linearity of expectation, for any coupling $(M_i, M_i')$ with the same marginals as $(N_i, N_i')$, we have  $\EE \left[ \log (N_i + 1) \right] -  \EE\left[ \log (N'_i + 1)  \right] = \EE \left[ \log \frac{M_i + 1}{M'_i + 1}  \right] \le \EE\left[\frac{M_i - M'_i}{M'_i + 1} \right].$ The rest of the proof focuses on designing a careful coupling between binomial and Poisson random variables with the same mean to obtain bounds on the above quantity, which can be of independent interest. Our proof also requires the following concentration bound, which we prove in Appendix~\ref{app:poi}.
\begin{lemma}
\label{lem:poi}
If $N \sim \operatorname{Poi}(\lambda)$, then with probability at least $1 - \delta$,
\[
|N + 1 - \lambda| \leq 6\sqrt{ (N+1)} \log(2/\delta).
\]
\end{lemma}
\section{Analysis}
We first relate the concentration bound for $\KL(p \,\|\,  \hat{p}^1)$ under multinomial sampling to $\widetilde{\KL}(p\,\|\,  \hat{p}')$ under Poisson sampling.
\begin{lemma}%
\label{lem:mult_to_poisson}
Let $\hat{p}'$ be measure given by $\hat{p}'_i = \frac{N'_i + 1}{n+k}$, where $N'_1, \ldots, N'_k \sim \text{Mult}(N, p_i)$ and $N \sim \operatorname{Poi}(n p_i)$. Then
\begin{align*}
\conf{&}\Pr(\KL(p \,\|\,  \hat{p}^1) 
 \geq \EE[\KL(p \,\|\,  \hat{p}^1)] + t ) 
\conf{\\&} \leq 3\sqrt{n} \Pr(\widetilde{\KL}(p \,\|\,  \hat{p}') \geq \EE[\widetilde{\KL}(p \,\|\,  \hat{p}')] + t - \gamma ),
 \end{align*}
 where $\gamma = \frac{311}{n} + \frc{160k}{n^{3/2}}$.
\end{lemma}
\begin{proof}
Since both $\sum_i p_i = 1$ and $\sum_i \hat{p}^1_i$= 1, we have
\begin{align*}
    \KL(p \,\|\,  \hat{p}^1 ) = \widetilde{\KL}(p \,\|\,  \hat{p}^1)   + \frac{n+k}{n} 
    - \log  \frac{n+k}{n}  - 1.
\end{align*}
Hence,
\[
  \KL(p \,\|\,  \hat{p}^1)  -    \EE[\KL(p \,\|\,  \hat{p}^1)]  = \widetilde{\KL}(p \,\|\,  \hat{p}^1) 
 - \EE[\widetilde{\KL}(p \,\|\,  \hat{p}^1)].
\]
Then we have
\begin{align*}
\conf{&}\Pr(\KL(p \,\|\,  \hat{p}^1) 
 \geq \EE[\KL(p \,\|\,  \hat{p}^1)] + t ) 
\conf{\\&}=
\Pr(\widetilde{\KL}(p \,\|\,  \hat{p}^1) 
 \geq \EE[\widetilde{\KL}(p \,\|\,  \hat{p}^1)] + t ).
 \end{align*}
Let $\hat{p}'$ be measure given by $\hat{p}'_i = \frac{N'_i + 1}{n+k}$, where $N'_1, \ldots, N'_k \sim \text{Mult}(N, p_i)$ and $N \sim \operatorname{Poi}(n p_i)$, then by Lemma~\ref{lem:KL_expectation},
\begin{align*}
\conf{&}\Pr(\widetilde{\KL}(p \,\|\,  \hat{p}^1) 
 \geq \EE[\widetilde{\KL}(p \,\|\,  \hat{p}^1)] + t ) 
\conf{\\&}  \leq \Pr(\widetilde{\KL}(p \,\|\,  \hat{p}^1) \geq \EE[\widetilde{\KL}(p \,\|\,  \hat{p}')] + t - \gamma ),
 \end{align*}
 where $\gamma = \frac{311}{n} + \frc{160k}{n^{3/2}}$.
 Furthermore observe that with probability at least $1/3\sqrt{n}$, we have $N=n$. Hence,
 \begin{align*}
\conf{&} \Pr(\widetilde{\KL}(p \,\|\,  \hat{p}^1) \geq \EE[\widetilde{\KL}(p \,\|\,  \hat{p}')] + t - \gamma )
    \conf{\\&} \leq 3\sqrt{n} \Pr(\widetilde{\KL}(p \,\|\,  \hat{p}') \geq \EE[\widetilde{\KL}(p \,\|\,  \hat{p}')] + t - \gamma ).
\end{align*}
Combining the above equations, yields the result.
\end{proof}
We now have all the tools needed for the result.
\begin{theorem}
\label{thm:KL}
With probability at least $1-\delta$, $\KL(p \,\|\,  \hat{p}^1 ) $ is upper bounded by
\[
 \EE[\KL(p \,\|\,  \hat{p}^1 ) ] + \frac{6\sqrt{k \log^5 (4k/\delta)}}{n} + \frac{311}{n} + \frac{160k}{n^{3/2}}.
\]
\end{theorem}
\begin{proof}
By Lemma~\ref{lem:mult_to_poisson}, if $\hat{p}'$ denotes the measure given by $\hat{p}'_i = \frac{N'_i + 1}{n+k}$, where $N'_1, \ldots, N'_k \sim \text{Mult}(N, p_i)$ and $N \sim \operatorname{Poi}(n p_i)$, we have
\begin{align}
\conf{&}\Pr(\KL(p \,\|\,  \hat{p}^1) 
 \geq \EE[\KL(p \,\|\,  \hat{p}^1)] + t ) 
\conf{\\&} \leq 3\sqrt{n} \Pr(\widetilde{\KL}(p \,\|\,  \hat{p}') \geq \EE[\widetilde{\KL}(p \,\|\,  \hat{p}')] + t - \gamma ), \label{eq:the:above:equation}
 \end{align}
 where $\gamma = \frac{311}{n} + \frc{160k}{n^{3/2}}$.
Note that we have both $\sum_i p_i = 1$ and $\sum_i \hat{p}^1_i$= 1.  We now bound the right-hand side in the above equation~\eqref{eq:the:above:equation}. We can write $\widetilde{\KL}(p \,\|\,  \hat{p}')$ as
\begin{align*}
\widetilde{\KL}(p \,\|\,  \hat{p}') = 
\sum_{i=1}^\ab \left(
p_i \log \frac{np_i}{N'_i+1} + \frac{N'_i+1}{n}  - p_i \right) .
\end{align*}
We first observe that each term here is bounded with high probability. By Jensen's inequality,
\begin{align}
\conf{&} p_i \log \frac{np_i}{N'_i+1} + \frac{N'_i+1}{n}  - p_i 
\conf{\nonumber \\} & = - p_i \log \frac{N'_i+1}{np_i} + \frac{N'_i+1}{n}  - p_i  \nonumber \\
& \geq - p_i \left( \frac{N'_i+1}{np_i} - 1 \right) + \frac{N'_i+1}{n}  - p_i 
\arxiv{\nonumber \\&} \geq 0.  \label{eq:temp1}
\end{align}
In the other direction,
\begin{align}
\conf{&}p_i \log \frac{np_i}{N'_i+1}  + \frac{N'_i+1}{n}  - p_i 
\conf{\nonumber\\}& \leq 
p_i \left(\frac{np_i}{N'_i+1} - 1 \right)
+\frac{N'_i+1}{n}  - p_i \nonumber  \\
& =  \left( \frac{\sqrt{n}p_i}{\sqrt{N'_i+1}} - \sqrt{\frac{N'_i+1}{n}}\right)^2  = \frac{\left(np_i - (N'_i+1) \right)^2}{n(N'_i+1)}. \nonumber
\end{align}
Hence by Lemma~\ref{lem:poi} and a union bound, for all $i$, with probability at least $1-\delta/2$,
\begin{equation}
\label{eq:temp2}
 p_i \log \frac{np_i}{N'_i+1} + \frac{N'_i+1}{n}  - p_i  
\leq \frac{36 \log^2(4k/\delta)}{n}.
\end{equation}
Let $\alpha \eqdef \frac{36 \log^2(4k/\delta)}{n}$. By~\eqref{eq:temp1} and~\eqref{eq:temp2}, 
with probability at least $1-\delta/2$, 
\begin{align*}
\widetilde{\KL}(p \,\|\,  \hat{p}')
= \sum^k_{i=1} \min\!\left( p_i \log \frac{np_i}{N_i+1} + \frac{N_i+1}{n}  - p_i, \alpha \right).
\end{align*}
Let $G = \sum^k_{i=1} \min\!\left( p_i \log \frac{np_i}{N_i+1} + \frac{N_i+1}{n}  - p_i, \alpha \right)$.
Changing one of the counts, changes $G$ by at most by $\alpha$. Hence, by McDiarmid's inequality, with probability at least $1-\delta/2$,
\vspace{-0.5ex}
\[
G \leq \EE[G] + \sqrt{\frac{n\alpha^2}{2}\log \frac{2}{\delta}}.
\]
\vspace{-0.5ex}
The lemma follows by observing that $ \EE[G] \leq \EE[[\widetilde{\KL}(p \,\|\,  \hat{p}')]$
and the union bound.
\end{proof}

\section {Lower bound}
    \label{sec:variance:lb}
In this section, we provide lower bounds on the variance of the KL divergence of the Laplace estimator thus showing that the bounds obtained in~\cref{thm:KL} bounds cannot be significantly improved (\cref{cor:sqrtk:lb}). In particular, we will prove the following variance lower bound for the Laplace estimator when the underlying distribution is uniform.

\begin{theorem}
    \label{th:variance:lb}
    Let $p$ be the uniform distribution over $[\ab]$ and $\hat{p}^1$ be the Laplace estimator applied on $\ns \geq 10\ab$ independent samples from $p$. Then,
    \[
        \var\!\Paren{\KL(p\,\|\,\hat{p}^1 )} \ge \frac{\ab}{32 \ns^2}. 
    \]
\end{theorem}

\begin{proof}
 When $p$ is the uniform distribution over $[\ab]$,
    \begin{align*}
        \KL(p \;\|\; \hat{p}^1) \conf{&} = \sum_{i \in [\ab]} \frac{1}{\ab} \log\Paren{\frac{\ns + \ab}{(N_i + 1)\ab}} \conf{\\&}=  - \frac{1}{\ab} \sum_{i \in [\ab]}  \log\Paren{N_i + 1} + C_{\ns,\ab},
    \end{align*}
    where $C_{\ns,\ab} \eqdef \log\left(1+\frac{n}{k}\right)$ is a constant independent of all $N_i$'s. Hence
    \begin{equation} \label{eqn:var_laplace}
        \var\Paren{\KL(p || \hat{p}^1 )} = \frac{1}{k^2} \var \Paren{\sum_{i \in [\ab]} \log\Paren{N_i + 1}}.
    \end{equation}
    Without loss of generality, assume $k = 2k'$ for an integer $k'$. Consider the event where we fix the sum of the following pairs $\{N_{2i-1} + N_{2i} \}_{i \in [k']}$. Let $f_i =\log\Paren{N_{2i-1} + 1}  + \log\Paren{N_{2i} + 1} $.    
    By law of total variance, we have
    \begin{align}
     \conf{&}   \var \Paren{\sum_{i \in [\ab]} \log\Paren{N_i + 1}}\conf{\nonumber \\} 
     & \ge \expectation{\var \Paren{\sum_{i = 1}^{k'} \Paren{ f_i } \mid \{N_{2i-1} + N_{2i} \}_{i \in [k']} } } \nonumber \\
        & = \sum_{i = 1}^{k'} \expectation{\var \Paren{f_i \mid N_{2i-1} + N_{2i}}} 
 \label{eqn:ind}\\
        & = \frac{\ab}{2}  \expectation{\var \Paren{\log\Paren{N_{1} + 1}  + \log\Paren{N_{2} + 1}\mid N_{1} + N_{2} } }.  \label{eqn:linear}
    \end{align}
    Here \eqref{eqn:ind} follows from the fact that conditioned on $N_{2i-1} + N_{2i}$, the $\log\Paren{N_{2i-1} + 1}   + \log\Paren{N_{2i} + 1}$ terms are independent. \eqref{eqn:linear} follows from that all $N_{2i-1} + N_{2i}$'s have the same distribution and the linearity of sum of expectations.

    Next we bound  $\expectation{\var \Paren{\log(N_{1}\! + \! 1) \!\! + \!\log(N_{2}\! + \!1)\mid N_{1} \! + \! N_{2}}  }$ for a fixed $N_{1} + N_{2} 
 = \nzero$. Since $p$ is a uniform distribution, when $N_{1} + N_{2} 
 = \nzero$, $N_1$ follows a Binomial distribution with $n$ trials and success probability $1/2$ and $N_2 = \nzero - N_1.$ Hence, we have
 \[
    \log\Paren{N_{1} + 1}  + \log\Paren{N_{2} + 1} =  \log \Paren{N_1(\nzero - N_1) + \nzero + 1}.
 \]
The next lemma will be helpful for bounding its variance.
\begin{lemma} \label{lem:var_f}
    If $f\colon [a, b] \to \RR$ is monotone and differentiable on $[a, b]$, for any random variable  $X$ supported on $[a, b]$, we have
    \[
        \var\Paren{f(X)} \ge \min_{x \in [a,b]} f'(x)^2 \cdot \var\Paren{X}.
    \]
\end{lemma}
\begin{proof}
    By continuity of $f$, there exists $x_0 \in [a,b]$ such that $f(x_0) = \expectation{f(X)}$. Hence we have
    \begin{align*}
        \var\Paren{f(X)}  &= \expectation{\Paren{f(X) - f(x_0)}^2}  \\& \ge \expectation{\min_{x \in [a,b]} f'(x)^2  \cdot \Paren{X - x_0}^2} \\&\ge \min_{x \in [a,b]} f'(x)^2 \cdot \var\Paren{X}. \qedhere
    \end{align*}
\end{proof}

Note that $N_1(\nzero - N_1) + \nzero + 1 \in [\nzero + 1, \frac{\nzero^2}{4} + \nzero + 1]$ and $\log'(x) = \frac{1}{x}$. By \cref{lem:var_f} and the distribution of $N_1$,
\begin{align*}
 \conf{&}   \var\Paren{\log \Paren{N_1(\nzero - N_1) + \nzero + 1}}  \conf{\\&}\ge \Paren{\frac{1}{\frac{\nzero^2}{4} + \nzero + 1}}^2 \cdot \var \Paren{N_1(\nzero - N_1)} \conf{\\&}= \frac{2\nzero(\nzero - 1)}{ (\nzero+2)^4},
\end{align*}
where we use $\var \Paren{N_1(\nzero - N_1)} = (\nzero^2 - \nzero)/8$. Hence 
\begin{align*}
   \conf{&}  \expectation{\var \Paren{\log\Paren{N_{1} + 1}  + \log\Paren{N_{2} + 1}\mid N_{1} + N_{2}}  } 
    \conf{\\&}  \ge \expectation{\frac{2(N_1 + N_2)(N_1 + N_2 - 1)}{(N_1 + N_2 + 2)^4}}
     \conf{\\&} \ge \expectation{\frac{\mathbf{1}\{N_1 + N_2 \ge 2\}}{4(N_1 + N_2 + 2)^2}}
\end{align*}
Since $N_1 + N_2 \sim \text{Binom}(\ns, 2/\ab)$, we have when $\ns \ge 10 \ab$,
\begin{align*}
  \conf{&}  \expectation{\var \Paren{\log\Paren{N_{1} + 1}  + \log\Paren{N_{2} + 1}\mid N_{1} + N_{2}}  } %
  \ge \frac{\ab^2}{16\ns^2}.
\end{align*}
Combining the above with \cref{eqn:var_laplace} and \cref{eqn:linear}, we complete the proof of the theorem.
\end{proof}

As a direct consequence of this lower bound, we obtain that the bound from~\cref{thm:KL} cannot be significantly improved: that is, that the $\sqrt{\ab}$ dependence is tight.
\begin{corollary}
    \label{cor:sqrtk:lb}
    There exist no constants $\eta>0$ and $C>0$ such that, for all $\delta\in(0,1]$, given $\ns\geq 10\ab$ samples, the Laplace estimator satisfies
\[
\KL(p \,\|\,  \hat{p}^1 ) \leq \EE[\KL(p \,\|\,  \hat{p}^1 ) ] + C\cdot \frac{k^{1/2-\eta}}{n}\log \frac{1}{\delta}
\]
with probability at least $1-\delta$.
\end{corollary}
\begin{proof}
    Assume by contradiction that such constants exist. We can rewrite $\var\Paren{\KL(p \| \hat{p}^1 )}$ as
    \begin{align*}
  &
  \mathbb{E}[\Paren{\KL(p \| \hat{p}^1 )- \mathbb{E}[\KL(p \| \hat{p}^1 )] }^2] \\
        &= \int_0^\infty \Pr\!\left[ \Paren{\KL(p \| \hat{p}^1 )- \mathbb{E}[\KL(p \| \hat{p}^1 )] }^2 > t \right] dt\\
        &\leq \int_0^\infty  e^{- \frac{\ns}{C\cdot \ab^{1/2-\eta}}\cdot \sqrt{t}} dt \tag{By assumption, for $\delta = e^{- \frac{\ns}{C\cdot \ab^{1/2-\eta}}}$}\\
        &= 2C^2\cdot \frac{k^{1-2\eta}}{n^2}
    \end{align*}
    which, for $\ab$ sufficiently large (as a function of $C$) contradicts the lower bound from~\cref{th:variance:lb}.
\end{proof}

\arxiv{
\begin{remark}
We note that the argument underlying the proof of~\cref{th:variance:lb} can be extended to obtain a lower bound in the regime $\ns \ll \ab$ as well, observing that in the case $N_1 + N_2 \sim \text{Binom}(\ns, 2/\ab)$ becomes approximately $\mathrm{Poisson}(2\ns/\ab)$, from which $\expectation{\frac{2(N_1 + N_2)(N_1 + N_2 - 1)}{(N_1 + N_2 + 2)^4}} = \Omega\!\left(\frac{\ns^2}{\ab^2}\right)$.
\end{remark}
}

\section{Conclusion}

We provided concentration bounds for the KL divergence between the underlying distribution and the Laplace estimator. Our results show that the dependence on the error probability can be bounded as $\tilde{O}(\sqrt{k} \log^{5/2}(1/\delta)/ n)$, which improves on the previous bound of $\tilde{O}(k \log(1/\delta)/ n)$ recently obtained by~\cite{BhattacharyyaGPV21}. We further established a lower bound of $\Omega(\sqrt{k}/n)$ on the variance and the tail bound of the KL loss of the Laplace estimator, thus showing our results are nearly-optimal. Extending our results to general add-$t$ estimators would be an interesting future research direction.

\arxiv{\section{Acknowledgments}

The authors thank Gautam Kamath and Ankit Pensia for helpful comments and discussions, and Yanjun Han and Yihong Wu for helpful pointers to the literature.
}

\bibliographystyle{plainnat}
\bibliography{references}
\appendix
\conf{\subsection{Constants in Figure~\ref{fig:samplevsestimate}}}
\arxiv{\section{Constants in Figure~\ref{fig:samplevsestimate}}}
\label{app:figure}

We heuristically compute the leading constant by doing asymptotic expansion of the KL divergence. Formalizing this heuristic might be an interesting future research direction.
By~\cite{kamath2015learning}, as $n \to \infty$ for the uniform distribution,
\begin{align*}
\KL(p \,\|\,  \hat{p}^1 ) & \approx \frac{1}{2} \sum^k_{i=1}\frac{(\hat{p}^1_i - p_i)^2}{\hat{p}^1_i}  \\
& \stackrel{\rm(a)}{\approx} \frac{1}{2} \sum^k_{i=1}  \frac{(\hat{p}^1_i - p_i)^2}{p_i}  \\
& =\frac{k}{2} \sum^k_{i=1} (\hat{p}^1_i - 1/k)^2 \\
&  \stackrel{\rm(b)}{\approx} \frac{k}{2} \sum^k_{i=1} (\hat{p}^\emp_i -  1/k)^2,
\end{align*}
where (a) relies on the fact that as $n\to \infty$, $\hat{p}^1_i \to p_i$ and (b) uses the fact that both Laplace and empirical estimator are similar as $n\to \infty$. Under normal approximation, $\hat{p}^\emp_i -  1/k$ is a normal distribution with mean $0$ and variance $1/(kn)$. Hence $\sum^k_{i=1} (\hat{p}^\emp_i -  1/k)^2$ can be approximated by a sum of squares of $k$ Gaussian random variables, which is a chi-squared distribution. Hence,
\begin{align*}
\Var(\KL(p \,\|\,  \hat{p}^1 )) 
& \approx \frac{k^2}{4} \cdot \frac{1}{(kn)^2} \Var(\chi^2_k) \\
& = \frac{k^2}{4} \cdot \frac{1}{(kn)^2} 2k \\
& = \frac{k}{2n^2},
\end{align*}
where $\chi^2_k$ is a standard Chi-squared random variable with $k$ degrees of freedom.

\conf{\subsection{Proof of Lemma~\ref{lem:KL_expectation}}}
\arxiv{\section{Proof of Lemma~\ref{lem:KL_expectation}}}
\label{app:KL_expectation}
By~\eqref{eq:tildeKL},
\begin{align*}
\conf{&}   \EE \left[\widetilde{\KL}(p \,\|\,  \hat{p}')\right]  
- \EE \left[\widetilde{\KL}(p \,\|\,  \hat{p}^1) \right] 
\conf{\\&} =   \EE \left[\KL(p \,\|\,  \hat{p}')\right]  
- \EE \left[\KL(p \,\|\,  \hat{p}^1) \right]  + \frac{n+k}{n}\EE \left[ \sum_i (\hat{p}'_i -  \hat{p}_i^1) \right].
\end{align*}
We first bound the last term:
\begin{align*}
   \EE \left[ \sum_i (\hat{p}'_i -  \hat{p}_i^1) \right] 
=  \EE \left[ \sum_i (\frac{N'_i+1}{n+k}-  \frac{N_i + 1}{n+k} )\right] = 0.
\end{align*}
Hence,
\begin{align*}
\conf{&}   \EE \left[\widetilde{\KL}(p \,\|\,  \hat{p}')\right]  
- \EE \left[\widetilde{\KL}(p \,\|\,  \hat{p}^1) \right] 
\conf{\\&}=   \EE \left[\KL(p \,\|\,  \hat{p}')\right]  
- \EE \left[\KL(p \,\|\,  \hat{p}^1) \right] 
\end{align*}
By definition,
\begin{align*}
\conf{&}\KL(p \,\|\,  \hat{p}') - \KL(p \,\|\,  \hat{p}^1) 
\conf{\\}& = 
\sum^k_{i=1} p_i \left( \log ({N_i + 1})
-  \log ({N'_i + 1}) \right).
\end{align*}
Let $M_i$ and $M'_i$ be any random variables which have the same marginals as $N_i$ and $N'_i$ respectively, but need not be independent. Then, for a given $i$, this can be further written as
\begin{align*}
\conf{&}  \EE\left[   \log ({N_i + 1})
-  \log ({N'_i + 1}) \right]
\conf{\\}& =  \EE\left[   \log ({M_i + 1})
-  \log ({M'_i + 1}) \right] \\
& = \EE\left[   \log \frac{M_i + 1}{M'_i + 1}  \right] \\
& \leq \EE \left[\frac{M_i - M'_i}{M'_i + 1} \right],
\end{align*}
where the last inequality follows from the fact that $\log (1+x) \leq x$. Hence,
\begin{align*}
\KL(p \,\|\,  \hat{p'}) - \KL(p \,\|\,  \hat{p}^1) 
& = \sum^k_{i=1} p_i \EE \left[ \log \frac{M_i + 1}{M'_i + 1}  \right] \conf{\\&}\leq \sum_i p_i \EE\left[\frac{M_i - M'_i}{M'_i + 1} \right].
\end{align*}
This holds for any joint distribution (coupling) of $(M_i,M'_i)$ with the right marginals, and so we can choose any such coupling that lets us derive a tight enough bound. We now define such a convenient coupling. Let $N \sim \text{Poi}(n)$. Let $X_i \sim \Bin(\min(N, n), p_i)$ and $Y_i\sim \Bin(|n-N|, p_i)$ be independent random variables. If $N > n$, let $
M_i = X_i$ and $M'_i = X_i + Y_i$. If $N \leq n$, let $M'_i = X_i$ and $M_i = X_i + Y_i$. That is,
\begin{align*}
    M_i &= X_i + Y_i \indic{N \leq n}\\
    M'_i &= X_i + Y_i \indic{N > n}
\end{align*}
With these definitions, one can check that, indeed, the marginals are correct:
\begin{itemize}
    \item $M'_i$ is a Poisson random variable with mean $n p_i$.
    \item $M_i$ is a Binomial random variable with parameters $n$ and $p$.
\end{itemize}
We now upper bound the above equation further:
\begin{align*}
\conf{&}    \frac{M_i - M'_i}{M'_i + 1}
\conf{\\}    & = \frac{M_i - M'_i}{X_i + 1} + \frac{M_i - M'_i}{M'_i + 1} - \frac{M_i - M'_i}{X_i + 1} \\
      & = \frac{M_i - M'_i}{X_i + 1} + \frac{(M_i - M'_i)(X_i -M'_i) }{(M'_i + 1)(X_i + 1)}\\
      &= \frac{Y_i (1-2\indic{N > n})}{X_i + 1} + \frac{-Y_i^2 (1-2\indic{N > n})\indic{N > n}}{(M'_i + 1)(X_i + 1)}\\
      &= \frac{Y_i (1-2\indic{N > n})}{X_i + 1} + \frac{Y_i^2\indic{N > n}}{(M'_i + 1)(X_i + 1)}\\
      &\leq \frac{Y_i (1-2\indic{N > n})}{X_i + 1} + \frac{Y_i^2\indic{N > n}}{(X_i + 1)^2} \tag{as $M'_i \geq X_i$} \\
      &\leq \frac{Y_i (1-2\indic{N > n})}{X_i + 1} + \frac{2Y_i^2}{(X_i + 1)(X_i+2)}.
\end{align*}
To proceed, we will rely on the following standard fact:
\begin{fact}
If $X\sim\Bin(m,p)$, then $\EE \left[\frac{1}{X+1}\right] = \frac{1-(1-p)^{m+1}}{p(m+1)}$ and
$\EE \left[\frac{1}{(X+1)(X+2)}\right] \leq \frac{1}{p^2(m+1)(m+2)}$.
\end{fact}
Using this and independence of $X_i$ and $Y_i$ conditioned on $N$, we get, focusing on the first term:
\begin{align}
\conf{&}    \EE \left[\frac{Y_i (1-2\indic{N > n})}{X_i + 1} \mid N \right] 
 \conf{\nonumber \\}   &= \frac{|N-n|p_i (1-2\indic{N > n})}{p_i(\min(N,n)+1)}\left(1- (1-p_i)^{\min(N,n)+1} \right) \notag\\
    &= \frac{n-N}{\min(N,n)+1}\left(1- (1-p_i)^{\min(N,n)+1} \right) \notag\\
    &\leq \frac{n-N}{\min(N,n)+1}\left(1- (1-p_i)^{n+1} \right)
\end{align}
where the last inequality follows from checking the two cases $N\geq n$ and $N< n$ separately.
Therefore, 
\begin{align}
 \conf{&}   \EE \left[\frac{Y_i (1-2\indic{N > n})}{X_i + 1} \right]
 \conf{\nonumber \\&}   = \EE \left[\EE \left[\frac{Y_i (1-2\indic{N > n})}{X_i + 1} \mid N \right]\right]
  \conf{\\}    &\leq \EE \left[\frac{n-N}{\min(N,n)+1}\right]  \left(1- (1-p_i)^{n+1} \right) \label{eq:first:term}
\end{align}
Similarly,
\begin{align*}
 \conf{&} \EE \left[ \frac{Y^2_i }{(X_i + 1)(X_i + 2)} \mid N \right] 
 \conf{\\} & \leq \frac{|n-N|^2 p^2_i + |n-N| p_i}{(\min(n, N)+1)(\min(n, N)+2)p^2_i}
 \conf{\\&} \leq \frac{|n-N|^2 p_i + |n-N| }{(\min(n, N)+1)^2p_i}.
\end{align*}
Combining the above set of equations yields,
\begin{align*}
 \conf{&}\EE \left[ \frac{M_i - M'_i}{M'_i + 1} \right]
\leq \EE \left[ \frac{(n - N)}{(\min(n, N) + 1)} \conf{\right]\\&} +  \conf{\EE \left[ }\frac{|n-N|^2}{(\min(n, N)+1)^2} +   \frac{|n-N|}{(\min(n, N)+1)^2p_i}\right].
\end{align*}
We now bound each of the three terms. We start with the last term.
\begin{align*}
\conf{&}  \EE\left[  \frac{|n-N|}{(\min(n, N)+1)^2p_i} \right]
  \conf{\\} & = \EE\left[  \indic{N \geq n/2} \frac{|n-N|}{(\min(n, N)+1)^2p_i}  \conf{\right] \\&} +  \conf{\EE\left[ } \indic{N < n/2} \frac{|n-N|}{(\min(n, N)+1)^2p_i} \right] \\
  & \leq \EE\left[ \frac{4|n-N|}{n^2p_i} + \frac{n\indic{N < n/2}}{p_i} \right] \\
  & \leq \frac{4\sqrt{n}}{n^2p_i} + \frac{n}{p_i} \Pr[ N < n/2 ] \\
  & \leq \frac{4}{n^{3/2}p_i} + \frac{e^{-n/8} n}{p_i}  \leq \frac{160}{n^{3/2}p_i}.
\end{align*}
We now bound the second term.
\begin{align*}
 \conf{&}    \EE \left[ \frac{|n-N|^2}{(\min(n, N)+1)^2} \right]
    \conf{\\} & =   \EE \left[  \indic{N \geq n/2}\frac{|n-N|^2}{(\min(n, N)+1)^2} \right] \conf{\\&} +   \EE \left[ \indic{N < n/2} \frac{|n-N|^2}{(\min(n, N)+1)^2} \right] \\
    & \leq \EE \left[ \frac{4(n-N)^2}{n^2} + \indic{N < n/2} (n-N)^2 \right] \\
    & \leq \frac{4}{n} + e^{-n/16} \sqrt{\EE \left[(n-N)^4\right]} \\
    & \leq \frac{4}{n} + 2e^{-n/16}n \leq \frac{282}{n}.
\end{align*}
We finally bound the first term.
\begin{align*}
\conf{&}\EE \left[ \frac{n-N}{\min(n, N) + 1} \right]   
\conf{\\}& = \EE \left[ \frac{n-N}{n + 1} +
\frac{n-N}{\min(n, N) + 1} -\frac{n-N}{n + 1} \right] \\
& = \EE \left[ %
\frac{(n-N)(n - \min(n, N))}{(\min(n, N) + 1)(n+1)} \right] \\
& \leq \EE \left[ 
\frac{(n-N)^2}{(\min(n, N) + 1)(n+1)} \right] \\
& = \EE \left[ \indic{N \geq n/2}\frac{(n-N)^2}{(\min(n, N) + 1)(n+1)} \right]  \conf{\\&} +  \EE \left[ \indic{N < n/2}\frac{(n-N)^2}{(\min(n, N) + 1)(n+1)} \right] \\
& \leq \EE \left[ \frac{2(n-N)^2}{n^2} + \indic{N < n/2} |N-n| \right] \\
& \leq \EE \left[ \frac{2(n-N)^2}{n^2} + \indic{N < n/2} |N-n| \right] \\
& \leq \frac{2}{n} + e^{-n/16} \sqrt{\EE \left[(N-n)^2\right]} \\
& \leq \frac{29}{n}\,,
\end{align*}
and therefore, recalling~\eqref{eq:first:term},
\begin{align}
    \EE \left[\frac{Y_i (1-2\indic{N > n})}{X_i + 1} \right]
    &\leq \left(1- (1-p_i)^{n+1} \right)\cdot \frac{29}{n}
    \leq \frac{29}{n}
\end{align}
Combining the above set of equations, we get 
\begin{align*}
\EE \left[ \frac{M_i - M'_i}{M'_i + 1} \right] \leq \frac{311}{n} + \frac{160}{n^{3/2}p_i}.
    \end{align*}
    Hence,
    \begin{align*}
\KL(p \,\|\,  \hat{p'}) - \KL(p \,\|\,  \hat{p}^1) 
 & = \sum^k_{i=1} p_i \log \frac{N_i + 1}{N'_i + 1}  \\ 
 &\leq \sum_i p_i \frac{N_i - N'_i}{N'_i + 1} \\
 & \leq \sum_i p_i \left(\frac{311}{n} + \frac{160}{n^{3/2}p_i} \right)\\
 & = \frac{311}{n} + \frac{160k}{n^{3/2}}.
\end{align*}

\conf{\subsection{Proof of Lemma~\ref{lem:poi}}}
\arxiv{\section{Proof of Lemma~\ref{lem:poi}}}
\label{app:poi}

By Poisson tail bounds \cite[Fact 12]{acharya2012competitive}, with probability at least $1-\delta/2$,
\[
|N - \lambda| \leq \sqrt{2 \max(N, \lambda) \log(2/\delta)}.
\]
Thus,
\[
|N + 1 - \lambda| \leq \sqrt{3 \max(N+1, \lambda) \log(2/\delta)}.
\]
If $N +1 \geq \lambda$,
\[
|N + 1 - \lambda| \leq \sqrt{3 (N+1) \log(2/\delta)}.
\]
If $N + 1\leq \lambda$,
\[
|N + 1 - \lambda| \leq \sqrt{3 \lambda \log(2/\delta)}.
\]
Let $\alpha = \sqrt{3 \log(2/\delta)/4}$. Then,
\[
(\sqrt{\lambda} - \alpha)^2 \leq \alpha^2 + N+ 1.
\]
Therefore, 
\[
\sqrt{\lambda} \leq 2\alpha + \sqrt{N+1}.
\]
Hence,
\begin{align*}
|N + 1 - \lambda| \conf{&}\leq 3\sqrt{ (N+1) \log(2/\delta)} + 3 \log(2/\delta)
\conf{\\&}\leq 6\sqrt{ (N+1)} \log(2/\delta).
\end{align*}
\end{document}